%% file: main.tex
\documentclass[letterpaper, 10 pt, conference]{ieeeconf}  

\IEEEoverridecommandlockouts                              

\usepackage{amsmath}
\usepackage{amsthm}
\usepackage{algorithm}
\usepackage[noend]{algpseudocode}
\usepackage{amssymb}  
\usepackage{xcolor}
\usepackage{graphicx}
\usepackage{booktabs}
\usepackage{hyperref}
\usepackage{mathtools}
\usepackage{multirow}
\usepackage{siunitx}
\usepackage{cite}
\input{math_commands}
\graphicspath{{figures/}}

\algnewcommand\And{\textbf{and} }
\algnewcommand\Or{\textbf{or} }

\title{\LARGE \bf
Receding-Constraint Model Predictive Control using a \\
Learned Approximate Control-Invariant Set
}

\author{Gianni Lunardi$^{1}$, Asia La Rocca$^{1}$, Matteo Saveriano$^{1}$ and Andrea Del Prete$^{1}$
\thanks{$^{1}$The authors are with the Industrial Engineering Department, University of Trento, Via Sommarive 11, 38123, Trento, Italy. {\tt \{name.surname\}@unitn.it}}%
\thanks{Funded by the European Union under NextGenerationEU, PRIN 2022 project STARLIT Prot. n. 2022ZE9J9J (CUP n. E53D23001130006), PRIN project ARIEL (CUP n. E53D23001020001), and project INVERSE under Grant Agreement n. 101136067.}
\thanks{\textcopyright2024 IEEE. Personal use of this material is permitted. Permission from IEEE must be obtained for all other uses, in any current or future media, including reprinting/republishing this material for advertising or promotional purposes, creating new collective works, for resale or redistribution to servers or lists, or reuse of any copyrighted component of this work in other works.}
}

\begin{document}

\maketitle

\thispagestyle{empty}
\pagestyle{empty}

\begin{abstract}
In recent years, advanced model-based and data-driven control methods are unlocking the potential of complex robotics systems, and we can expect this trend to continue at an exponential rate in the near future. 
However, ensuring safety with these advanced control methods remains a challenge. 
A well-known tool to make controllers (either Model Predictive Controllers or Reinforcement Learning policies) safe, is the so-called \emph{control-invariant set} (a.k.a. safe set).
Unfortunately, for nonlinear systems, such a set cannot be exactly computed in general. 
Numerical algorithms exist for computing approximate control-invariant sets, but classic theoretic control methods break down if the set is not exact.
This paper presents our recent efforts to address this issue. 
We present a novel Model Predictive Control scheme that can guarantee recursive feasibility and/or safety under weaker assumptions than classic methods. 
In particular, recursive feasibility is guaranteed by making the safe-set constraint move backward over the horizon, and assuming that such set satisfies a condition that is weaker than control invariance.
Safety is instead guaranteed under an even weaker assumption on the safe set, triggering a safe task-abortion strategy whenever a risk of constraint violation is detected.
We evaluated our approach on a simulated robot manipulator, empirically demonstrating that it leads to less constraint violations than state-of-the-art approaches, while retaining reasonable performance in terms of tracking cost, number of completed tasks, and computation time.
\end{abstract}

\section{Introduction}\label{sec:introduction}
Ensuring safety is crucial in all robotics applications. 
However, this is more and more difficult with the recently increasing complexity of control methods and robotic platforms.
Indeed, recent data-driven approaches, often relying on Reinforcement Learning (RL) algorithms, typically produce black-box policies that are inherently hard to certify as safe. 
Moreover, even model-based control methods for constrained nonlinear systems in practice struggle to guarantee safety, which consists in recursive constraint satisfaction (a.k.a. \textit{recursive feasibility}). 
This is because the classic approach to guaranteeing safety, both for Model Predictive Control (MPC) and for Quadratic-Programming-based control methods, relies on the assumption of knowing a so-called \emph{safe set} (a.k.a. control-invariant set)~\cite{Blanchini1999, Grune2017}, or a Control Barrier Function (CBF)~\cite{Ames2014, Wu2019}. 
However, exactly computing safe sets (or CBFs) for nonlinear systems is not feasible in general. 
Therefore, practitioners must rely on numerical methods to compute approximate versions of such sets (or functions)~\cite{Djeridane2006, Coquelin2007,Jiang2016,Rubies2016,Hsu2021,Dawson2023, Zhou2020, LaRocca2023}.
Unfortunately, safety guarantees are lost if the safe set is not exact.
A more detailed discussion of state-of-the-art MPC methods is postponed to Section~\ref{ssec:recursive_feasibility} and~\ref{ssec:soft_terminal_constraint}.

An alternative approach is \textit{Model Predictive Shielding} (MPS)~\cite{Li2020a, Bastani2021}, which relies on a backup policy that drives the state to an invariant set.
This set can be arbitrarily small, as long as the policy can reach it from a large set.
Contrary to a safe set, a backup policy does not need to be certified because its ability to drive the state to a safe set is verified at every control loop.
However, finding a valid backup policy and the associated (small) invariant set has a similar complexity to finding a (large) safe set.

In this paper, we present a novel MPC scheme that ensures: i) safety, assuming the safe set is a \emph{conservative} approximation of a specific backward reachable set; ii) recursive feasibility, assuming the safe set is N-step control invariant, which is a weaker assumption than classic control invariance.
We compared our approach with classic MPC schemes: the standard formulation (without terminal constraints but a longer horizon), and a formulation using the safe set to constrain the terminal state. 
Our method could successfully avoid constraint violation in more tests than the others, being able to trade off performance and safety depending on the conservativeness of the used safe set.

\section{Preliminaries}
This section introduces the problem of \emph{safety} in MPC, together with the concepts of control-invariant set and Recursive Feasibility.

\subsection{Notation}
\begin{itemize}
    \item $\mathbb{N}$ denotes the set of natural numbers;
    \item $\{ x_i \}_0^N$ denotes a discrete-time trajectory given by the sequence $(x_0, \dots, x_N)$;
    \item $x_{i|k}$ denotes the state at time step $k+i$ predicted when solving the MPC problem at time step $k$;
\end{itemize}

\subsection{Problem statement}
Let us consider a discrete-time dynamical system with state and control constraints:
\begin{equation} \label{eq:sys_dynamics}
    x_{i+1} = f(x_i, u_i),
\qquad
    x \in \Xs, \qquad u \in \Us .
\end{equation}
Our goal is to design a control algorithm to ensure \emph{safety} (i.e., constraint satisfaction), while preserving performance (i.e., cost minimization) as much as possible.
Let us define $\mathcal{S}$ as the set containing all the equilibrium states of our system:
\begin{equation}
    \mathcal{S} = \{ x \in \Xs\ | \, \exists \, u \in \Us : x = f(x,u) \}.
\end{equation}
To achieve our goal, we rely on the \emph{Infinite-Time Backward-Reachable Set}~\cite{Blanchini1999} of $\mathcal{S}$, which we denote as \Vs. 
Mathematically, it is defined as the subset of \Xs\ starting from which it is possible to reach $\mathcal{S}$ in finite time:
\begin{equation}
\begin{aligned}
\Vs \triangleq \{ x_0 \in \Xs \, | \, \exists \{u_i\}_{0}^{k}, k \in \mathbb{N}: \, \,
& x_{k+1} \in \mathcal{S}, x_i \in \Xs, \\
& u_i \in \Us, \forall \, i=0,\dots,k \} .
\end{aligned}
\end{equation}
As all backward reachable sets of equilibrium states, the set $\Vs$ is a control-invariant set~\cite{Blanchini1999}.
This means that, starting from inside \Vs, it is possible to remain inside \Vs\ indefinitely.
If we knew \Vs\ we could use it to construct a safe controller. 
However, we cannot reasonably assume to know it in general, but we rely instead on a more realistic assumption.
\begin{assumption} \label{ass:conservative_set}
    We know a conservative approximation of the set \Vs : 
    \begin{equation}
        \hat{\Vs} \subseteq \Vs
    \end{equation}
    Note that $\hat \Vs$ is not control invariant in general.
\end{assumption}
\begin{assumption} \label{ass:time_steps}
    We know an upper bound on the number of time steps needed to safely drive the system to an equilibrium from a state in $\hat \Vs$, which we refer to as $\bar{N}$.
\end{assumption}
As discussed in Section~\ref{sec:introduction}, numerical methods exists to compute approximations of \Vs. Among the others, we chose the Viability-Boundary Optimal Control (VBOC) method~\cite{LaRocca2023}, which finds a numerical approximation of $\Vs$ using states that are guaranteed to be on its boundary. The approximation can be made conservative by an appropriate choice of a safety margin. Additionally, VBOC produces also an estimate of $\bar{N}$, satisfying Assumption~\ref{ass:time_steps}.
Now we discuss different approaches to exploit $\hat{\Vs}$ in an MPC formulation to try to achieve safety.

\subsection{Model Predictive Control and Recursive Feasibility}
\label{ssec:recursive_feasibility}
Let us consider the following MPC problem:
\begin{subequations}
\begin{align}
        \minimize_{\{x_i\}_0^N,\{u_i\}_0^{N-1}}  &\quad \sum_{i = 0}^{N-1} \ell_i(x_i,u_i) + \ell_N(x_N)  \label{eq:mpc_cost} \\
        \st &\quad x_0 = x_{init} \label{eq:mpc_initial_conditions}\\
            &\quad x_{i+1} = f(x_i, u_i) \qquad i = 0 \dots N-1 \label{eq:mpc_dynamics}\\
            &\quad x_i \in \Xs, u_i \in \Us \qquad \;\; i = 0 \dots N-1 \label{eq:mpc_path_constraints}\\
            &\quad x_{N} \in \mathcal{X}_N, \label{eq:mpc_terminal_constraint}
\end{align}
\label{eq:mpc}
\end{subequations}
where $\ell(\cdot)$/$\ell_N(\cdot)$ is the running/terminal cost, $x_{init}$ is the current state, and $\mathcal{X}_N \subseteq \Xs$ is the terminal set~\cite{Mayne2000}.


Even though MPC is one of the most suited frameworks for controlling constrained systems, ensuring safety (i.e., constraint satisfaction) remains challenging when the dynamics or the constraints are nonlinear.
The most common approach to ensuring safety is based on Recursive Feasibility (RF), which guarantees that, under the assumption of no disturbances/modeling errors, if an MPC problem is feasible at the first loop, it remains feasible forever.

RF is guaranteed if the MPC horizon $N$ is \emph{sufficiently} long (see Section 8.2 of~\cite{Grune2017}). 
However, in general we cannot know how long $N$ should be.
Moreover, even if $N$ was known, it may be too long to result in acceptable computation times.

Alternatively, RF can be guaranteed by using the terminal set $\mathcal{X}_N$ to constrain the final state inside a \emph{control-invariant} set (see Section~\ref{ssec:soft_terminal_constraint}). 
While theoretically elegant, the practical issue with this approach is that control-invariant sets are extremely challenging (if not impossible) to compute for nonlinear systems/constraints.
A special case of this approach is when an \emph{equilibrium} state (or a set of equilibria) is used as terminal set.
This solves the issue of computing control-invariant sets, but at the price of (potentially drastically) reducing the \emph{basin of attraction} of the MPC.

Other approaches to RF exist that rely on the optimality properties of the solution and the stability of the closed loop (e.g., Section 8.3 of~\cite{Grune2017}).
However, these approaches require controllability and other conditions on running and terminal costs.
Therefore, they are not applicable to arbitrary cost formulations as the methods discussed in this paper.

\subsection{Soft Terminal Constraint}
\label{ssec:soft_terminal_constraint}
As discussed above, a common way to ensure recursive feasibility in MPC is to constrain the final state inside a control-invariant set, such as \Vs. 
Unfortunately, we do not know \Vs, but only $\hat{\Vs}$, which is not control invariant in general.
Therefore, using $\hat{\Vs}$ as terminal set in our MPC
does not ensure RF.
This means that our MPC problem could become unfeasible, and at that point classic MPC theory does not tell us what to do.
A common strategy to deal with unfeasibility is to relax the terminal constraint with a slack variable, which is heavily penalized in the cost function~\cite{Kerrigan2000soft, Zeilinger2014}.
In this way, when the terminal constraint cannot be satisfied, we can still get a solution that allows us to keep controlling the system, in the hope that perhaps the terminal constraint be satisfied again. 
However, this approach does not ensure \emph{safety}, nor RF,
because the soft constraint allows the state to leave $\hat{\Vs}$, which eventually can lead to constraint violations.

\section{Safe Model Predictive Control}
This section describes our novel MPC scheme, which relies on two components: a safe task-abortion strategy (Section~\ref{ssec:task_abortion}, and a receding-constraint MPC formulation (Section~\ref{ssec:receding_constraint}), which can  be used together (Section~\ref{ssec:abort_receding}).

\subsection{Safe Task Abortion} \label{ssec:task_abortion}
Our key idea to ensure safety relies on Assumption~\ref{ass:conservative_set} and \ref{ass:time_steps}, and on the following two assumptions.
\begin{assumption} \label{ass:computational_units}
We have access to two computational units, which we refer to as unit A and unit B. 
\end{assumption}
\begin{assumption} \label{ass:ocp_time}
We can solve the following OCP for any $x_{init} \in \hat{\Vs}$, in at most $N-1$ time steps:
\begin{equation}
    \begin{aligned}
        \minimize_{\{x_i\}_0^{\bar{N}},\{u_i\}_0^{\bar{N}-1}}  &\quad \sum_{i = 0}^{\bar{N}-1} \ell_i(x_i,u_i)  + \ell_{\bar{N}}(x_{\bar{N}})\\
        \st & \quad \eqref{eq:mpc_initial_conditions}, \eqref{eq:mpc_dynamics}, \eqref{eq:mpc_path_constraints} \\
            &\quad x_{\bar{N}} = x_{\bar{N}-1},
    \end{aligned}
    \label{eq:safe_abort_ocp}
\end{equation}
The choice of the cost function is irrelevant, and can simply be used to help the solver to converge faster.
\end{assumption}
Optimal Control Problem (OCP)~\eqref{eq:safe_abort_ocp} can be used to find a feasible trajectory to reach an equilibrium state from $x_{init}$.
Now we can describe our strategy to safely abort the task in case we detect a risk of constraint violation. 
Let us assume that we are using a classic MPC formulation with terminal constraint $x_N \in \hat{\Vs}$, and that at the MPC loop $k$ our problem becomes unfeasible. 
In this situation, we can follow these steps to safely abort the task:
\begin{enumerate}
    \item unit A uses the MPC solution computed at loop $k-1$ to reach the terminal state $x_{N|k-1} \in \hat{\Vs}$;
    \item in parallel, unit B solves  OCP~\eqref{eq:safe_abort_ocp}, using $x_{N|k-1}$ as initial state;
    \item after reaching $x_{N|k-1}$, we follow the solution of OCP~\eqref{eq:safe_abort_ocp} to safely reach an equilibrium state.
\end{enumerate}
This strategy allows us to reach a safe equilibrium state, where a stabilizing controller can be used to maintain the system still. 
Actually, we do not need to abort the task as soon as one MPC problem becomes unfeasible.
While we follow the last feasible solution, we can keep trying to solve OCP~\eqref{eq:mpc}. 
This strategy is summarized in Alg.~\ref{alg:terminal} and it can guarantee safety, as stated in the following Lemma.
\begin{lemma}
    Under Assumptions~\ref{ass:conservative_set} to~\ref{ass:ocp_time}, the hard terminal-constraint MPC with safe task abortion described in Alg.~\ref{alg:terminal} guarantees that constraints are never violated.
\end{lemma}
\begin{proof}
This proof is straightforward. 
OCP~\eqref{eq:safe_abort_ocp} is always feasible because, by Assumption~\ref{ass:conservative_set} and \ref{ass:time_steps}, from any state in $\hat{\Vs}$ we can reach an equilibrium in at most $\bar{N}$ time steps.
Assumption~\ref{ass:ocp_time} ensures that, by dedicating a computational unit to solving OCP~\eqref{eq:safe_abort_ocp}, we get a solution before reaching the terminal state of the last feasible MPC problem, $x_{N|k-1}$.
After reaching $x_{N|k-1}$, we follow the solution of OCP~\eqref{eq:safe_abort_ocp} to reach an equilibrium state, in which we can stay forever without violating the constraints.
\end{proof}
Our most critical assumption is probably Assumption~\ref{ass:ocp_time}, which relies on the MPC horizon $N$ to be sufficiently long, and on $\bar{N}$ not to be too large, to allow for enough computation time to solve the OCP.
This may be challenging because we can expect $\bar{N}$ to be rather large, since it must be sufficient to allow the system to reach an equilibrium from any state in $\hat{\Vs}$. 
At the same time, $N$ cannot be set too large because it is proportional to the computation time of the MPC problem.
However, this assumption turned out to be satisfied in our tests; if that was not the case, learning-based warm-start techniques could be used to speed-up computation~\cite{Mansard2018,Grandesso2023}.

\begin{algorithm}[t]
\caption{Terminal-Constraint MPC with Safe Abortion}
\small
\begin{algorithmic}[1]
\Require 
Number of time steps $T$,
Initial state $x_0$,
Initial guess $\{x^g_i\}_0^{N}, \{u^g_i\}_0^{N-1}$,
OCP~\eqref{eq:mpc},
Safe-abort OCP~\eqref{eq:safe_abort_ocp}
    \State $fails \leftarrow 0$ \Comment{Counter for failed OCP's}
    \For{$t = 0 \rightarrow T-1$}
        \State $\{x_i^*\}_0^{N}, \{u_i^*\}_0^{N-1}, feas \leftarrow \textsc{OCP}(x_t, \{x^g_i\}_0^{N}, \{u^g_i\}_0^{N-1})$
        \If{$feas$ = True} \Comment{If OCP's solution is feasible}
            \State $fails \leftarrow 0$ \Comment{Reset counter}
        \Else
            \If{$fails = 0$} \Comment{Start solving \eqref{eq:safe_abort_ocp} in Unit B}
                \State $\textsc{SolveSafeAbortOCP}(x_{N-1}^g)$
            \EndIf
            \If{$fails = N-1$} \Comment{Abort task}
                \State \Return $\textsc{FollowSafeAbortTrajectory()}$
            \EndIf
            \State $fails \leftarrow fails + 1$ \Comment{Increment counter}
            \State $\{x^*_i\}_0^{N}, \{u^*_i\}_0^{N-1} \leftarrow \{x_i^g\}_0^{N}, \{u_i^g\}_0^{N-1}$ \Comment{Copy last feasible solution}
        \EndIf
        \State $x_{t+1} \leftarrow f(x_t,u_0^*)$ \Comment{Simulate system}
        \State $\{x^g_i\}_0^{N-1}, \{u^g_i\}_0^{N-2} \leftarrow \{x_i^*\}_1^{N}, \{u_i^*\}_1^{N-1}$
        \State $x^g_N, u^g_{N-1} \leftarrow x^g_{N-1}, u^g_{N-2}$
    \EndFor
\end{algorithmic}
\label{alg:terminal}
\end{algorithm}

\subsection{Receding-Constraint MPC} 
\label{ssec:receding_constraint}
\begin{figure}[tbp]
\centering
\includegraphics[width=0.5\textwidth]{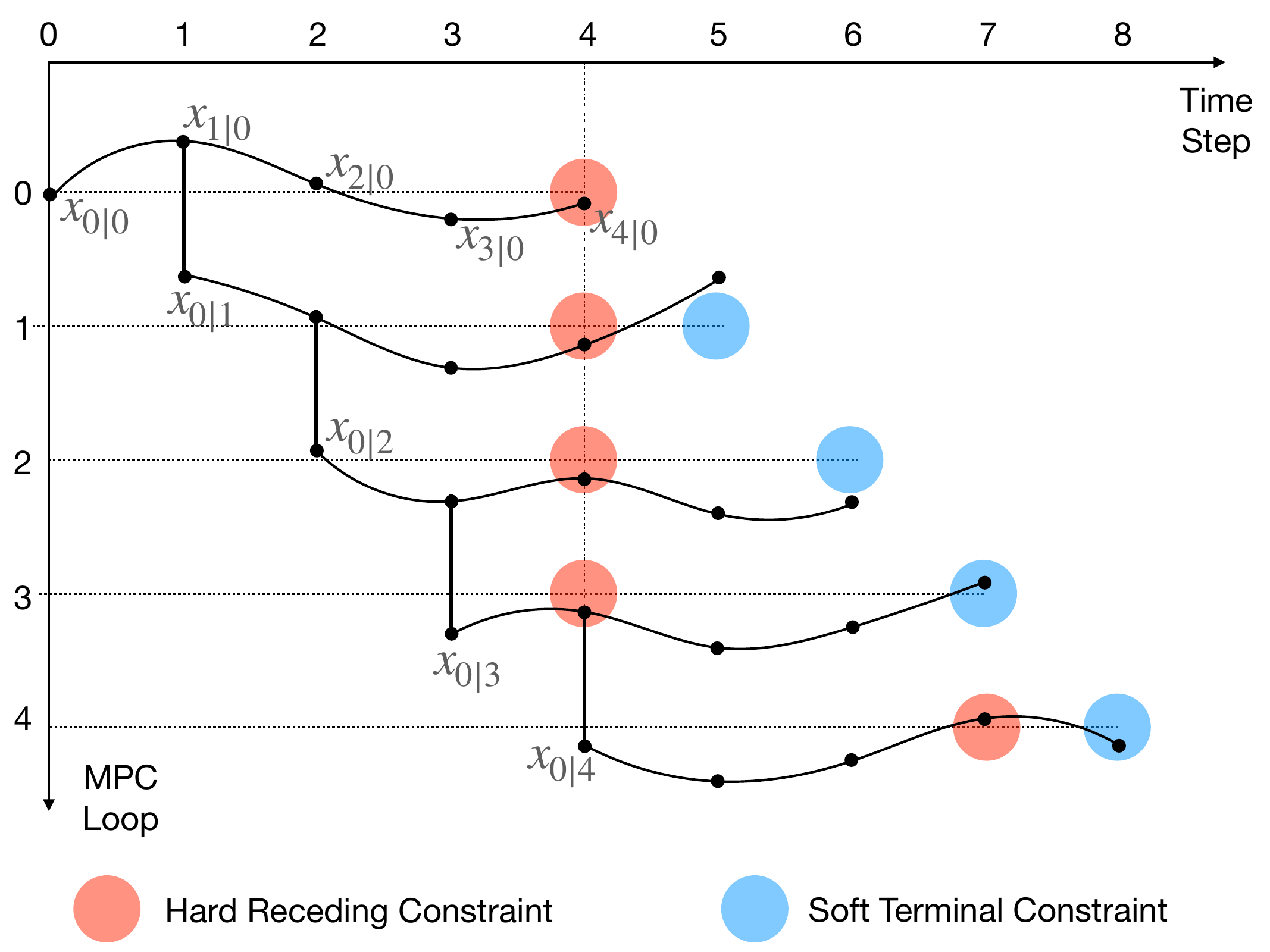}
\caption{Example of Receding-Constraint MPC with $N=4$. After the MPC loop 3, the receding constraint slides forward because $x_{4|3} \in \hat{\Vs}$.}
\label{fig:receding_constr_example}
\end{figure}

Instead of relying exclusively on the final state to ensure safety, we could exploit the fact that, as long as at least one state $x_r \in \hat{\Vs}$ (with $1 \le r \le N$), we know that $x_1 \in \Vs$ because from $x_1$ we can reach $x_r$. 
This suggests that a less conservative constraint to include in our OCP would be:
\begin{equation} \label{eq:intermediate_hard_constraints}
    (x_1 \in \hat{\Vs}) \, \lor \, (x_2 \in \hat{\Vs}) \lor \, \dots \, \lor \, (x_N \in \hat{\Vs})
\end{equation}
Unfortunately, OR constraints are extremely challenging for numerical solvers.
Even if this constraint cannot be used, we can find other ways to exploit this insight.

We suggest to adapt online the time step at which we constrain the state in $\hat{\Vs}$.
For instance, if at the MPC loop $k-1$ we had $x_{r|k-1} \in \hat{\Vs}$, at the loop $k$ we know that it is possible to have $x_{r-1|k} \in \hat{\Vs}$ (assuming no disturbances and modeling errors), therefore we can impose this constraint in a hard way.
This is sufficient to ensure safety for $r$ loops, during which this \emph{receding constraint} would slide backward along the horizon.
However, once the receding constraint reaches time step 0, we can no longer rely on it to ensure safety.
Therefore, we suggest to maintain also a soft constraint to encourage the terminal state to be in $\hat{\Vs}$.
This MPC formulation can be stated as:
\begin{equation}
    \begin{aligned}
        \minimize_{\{x_i\}_0^N,\{u_i\}_0^{N-1}, s}  &\quad \sum_{i = 0}^{N-1} \ell_i(x_i,u_i) + \ell_N(x_N) + w_s |s| \\
        \st &\quad \eqref{eq:mpc_initial_conditions}, \eqref{eq:mpc_dynamics}, \eqref{eq:mpc_path_constraints} \\
            & \quad x_r \in \hat{\Vs} \\
            &\quad x_{N} \in \hat{\Vs} \oplus s. \\
    \end{aligned}
    \label{eq:receding_constr_mpc}
\end{equation}
After solving the MPC at loop $k-1$, we can check whether $x_{N|k-1} \in \hat{\Vs}$; if that is the case, at loop $k$ we can move the receding constraint forward on $x_{N-1|k}$, which ensures safety for other $N-1$ loops. 
Actually, we can even check whether $x_{i} \in \hat{\Vs}$, for any $i>r$, and if that is the case we can set $r=i-1$ at the next loop.
The resulting algorithm is summarized in Alg.~\ref{alg:receding_with_abort} (with \texttt{SafeAbortFlag}  set to false), and a simple example is depicted in Fig.~\ref{fig:receding_constr_example}.



\begin{algorithm}[t]
\caption{Receding-Constraint MPC with Task Abortion}
\small
\begin{algorithmic}[1]
\Require 
Number of time steps $T$,
Initial state $x_0$,
\mbox{Initial guess $\{x^g_i\}_0^{N}, \{u^g_i\}_0^{N-1}$},
OCP~\eqref{eq:receding_constr_mpc},
\texttt{SafeAbortFlag}, 
Safe-abort OCP~\eqref{eq:safe_abort_ocp}
    \State $r \leftarrow N$  \Comment{Receding constraint index}
    \For{$t = 0 \rightarrow T-1$}
        \If{\texttt{SafeAbortFlag} \textbf{and} $r = 0$} \Comment{Abort task}
            \State \Return $\textsc{FollowSafeAbortTrajectory()}$
        \EndIf
        \State $\{x_i^*\}_0^{N}, \{u_i^*\}_0^{N-1} \leftarrow \textsc{OCP}(r, x_t, \{x^g_i\}_0^{N}, \{u^g_i\}_0^{N-1})$
        \State $r \gets r - 1$  \Comment{Recede constraint}
        \For{$k = r+2 \rightarrow N$} \Comment{Search last state in $\hat{\Vs}$}
            \If{$x_k^* \in \hat \Vs$}
                \State $r \leftarrow k - 1$
            \EndIf
        \EndFor
        \If{\texttt{SafeAbortFlag}  \textbf{and} $r = 0$} \Comment{Cannot recede}
            \State $\textsc{SolveSafeAbortOCP}(x_{1}^*)$ \Comment{Solve~\eqref{eq:safe_abort_ocp} in Unit B}
        \EndIf
        \State $x_{t+1} \leftarrow f(x_t,u_0^*)$ \Comment{Simulate dynamics}
        \State $\{x^g_i\}_0^{N-1}, \{u^g_i\}_0^{N-2} \gets \{x_i^*\}_1^{N}, \{u_i^*\}_1^{N-1}$
        \State $x^g_N, u^g_{N-1} \leftarrow x^g_{N-1}, u^g_{N-2}$
    \EndFor
\end{algorithmic}
\label{alg:receding_with_abort}
\end{algorithm}

To clarify the theoretical properties of this receding-constraint formulation, we first need to introduce the concept of $N$-step control invariant set.
\begin{definition}
    A set $\mathcal{A} \subseteq \Xs$ is \emph{$N$-step control invariant} if, starting from any state in $\mathcal{A}$, it is possible either to remain in $\mathcal{A}$, or to leave $\mathcal{A}$ and come back to it within $N$ time steps:
    \begin{equation} \begin{aligned} \label{eq:n_step_control_invariance}
        \forall x_0 \in \mathcal{A}: &\, 
        \exists \, \{ u_i \}_{0}^{k-1}, \,
        1 \le k \le N, \\
        & x_k \in \mathcal{A}, \,
        x_i \in \Xs, \,
        u_i \in \Us, \,
        \forall i=0, \dots, k-1
    \end{aligned} \end{equation}
\end{definition}
This is an extension of the well-known control invariance, with 1-step control invariance being equivalent to classic control invariance.
Now we can state under which conditions the receding-constraint MPC is recursively feasible.

\begin{theorem} \label{th:recursive_feasibility}
Assuming $\hat{\Vs}$ is $N$-step control invariant and the penalty on the soft terminal constraint $w_s$ is larger than the associated optimal Lagrange multiplier, the Receding-Constraint MPC formulation described in Alg.~\ref{alg:receding_with_abort} (with \emph{\texttt{SafeAbortFlag}}  set to false) is recursively feasible.
\end{theorem}
\begin{proof}
Assume the receding-constraint formulation is feasible at the first MPC loop $k=0$, which implies that $x_{N|0} \in \hat{\Vs}$.
This guarantees recursive feasibility for $N$ loops, during which the receding constraint can slide backward along the horizon.
However, since $\hat{\Vs}$ is $N$-step control invariant by assumption, we know that in one of those $N$ loops it will be possible to satisfy the soft terminal constraint $x_N \in \hat{\Vs}$.
This is because at each loop $k$, the MPC solver tries to satisfy condition~\eqref{eq:n_step_control_invariance} for a fixed value of $k$.
Since we know that~\eqref{eq:n_step_control_invariance} is feasible for some $k\in [1, N]$, we can infer that the soft terminal constraint $x_{N|k} \in \hat{\Vs}$ must be feasible for some MPC loop $k\in [1, N]$.
Under the assumption that $w_s$ is sufficiently large~\cite{nocedal1999numerical}, the used $l_1$ penalty function is exact, and therefore we can infer that at loop $k$, the soft terminal constraint will be exactly satisfied.
When this happens, the hard receding constraint moves to time step $N-1$, ensuring recursive feasibility for another $N-1$ loops.
At this point the same reasoning can be applied to ensure recursive feasibility indefinitely.
\end{proof}
This theorem highlights how the proposed receding-constraint MPC guarantees recursive feasibility even if the set $\hat{\Vs}$ is not control invariant. 
We rely indeed on a weaker condition, which is \emph{$N$-step control invariance}. 
Our condition is weaker because any $1$-step control invariant set is also $N$-step control invariant, for any $N>1$, therefore our approach guarantees recursive feasibility for a larger class of sets, which contains the class of control-invariant sets.
Unfortunately, computing exactly an $N$-step control invariant set is currently as hard as computing a standard control invariant set. 
However, in practice, it is more likely that a numerical method for approximating control invariant sets produces a set that is $N$-step control invariant, rather than control invariant. 
Therefore, as empirically shown in our results, our approach has a higher probability of being recursively feasible than a terminal-constraint MPC, even if in practice we cannot guarantee the assumptions of Theorem~\ref{th:recursive_feasibility} to be satisfied.


\subsection{Safe Task Abortion with Receding Constraint}
\label{ssec:abort_receding}
Since in practice we cannot guarantee that $\hat{\Vs}$ be $N$-step control invariant, we cannot guarantee that the receding constraint formulation be recursively feasible.
Therefore, we may need to use the task-abortion strategy when the receding constraint has reached time step 0. 
The problem is that at that point we have only one time step to solve OCP~\eqref{eq:safe_abort_ocp}.
In this paper, we assume that this computation time is enough, and we describe in Alg.~\ref{alg:receding_with_abort} (with \texttt{SafeAbortFlag} set to true) the Receding-Constraint MPC with Task Abortion. 

If one time step were not sufficient to solve~\eqref{eq:safe_abort_ocp}, several solutions could be explored.
We briefly discuss them in the following, but we leave their implementation for future work.
A possible way to reduce computation time is to pre-compute a warm-start for OCP~\eqref{eq:safe_abort_ocp}, before $r$ reaches 0. 
While we do not know in which state the system will be at that time, we can use the trajectory predicted by the MPC as a guess.
%
If this warm-start is not enough to solve OCP~\eqref{eq:safe_abort_ocp} in one time step, we could modify the receding-constraint formulation to ensure that the pre-computed \emph{safe-abort trajectory} starts exactly at the state of the system when the task abortion is initiated. 
To achieve this, we must modify the receding constraint from $x_{j|k} \in \hat{\Vs}$ to the more conservative $x_{j|k} = x_{j+1|k-1}$. 
In other words, we constrain the predicted state in $\Vs$ not to change across the MPC loops.
This is bound to deteriorate performance, but it should still outperform the standard Terminal-Constraint MPC.

\section{Simulation Results}
This section presents our simulation results\footnote{Our open-source code is freely available at \url{https://github.com/idra-lab/safe-mpc}.}.

\subsection{Simulation Setup}
To thoroughly evaluate the role of the \emph{safe abort} and the \emph{receding constraint}, we conducted an ablation study comparing five MPC formulations of increasing complexity:
\begin{itemize}
    \item \emph{Naive}: a classic formulation without terminal constraint, i.e.,  problem~\eqref{eq:mpc} with $\mathcal{X}_N = \Xs$. This is the baseline for all the experiments. 
    \item \emph{Soft Terminal} (\emph{ST}): it introduces a soft terminal constraint set $\mathcal{X}_N = \hat{\Vs}$ with a penalty weight of $10^4$, as a first step towards RF (Section~\ref{ssec:soft_terminal_constraint}).
    \item \emph{Soft Terminal With Abort} (\emph{STWA}): as the previous one, but it triggers the \emph{safe abort} whenever $x_{N|k} \notin \hat{\Vs}$ (as in Alg.~\ref{alg:terminal}).
    \item \emph{Hard Terminal With Abort} (\emph{HTWA}): it substitutes the soft terminal constraint of \emph{STWA} with a hard one, to satisfy the constraints of the theoretical OCP \eqref{eq:mpc}.
    \item \emph{Receding}: the novel formulation~\eqref{eq:receding_constr_mpc} described by Alg.~\ref{alg:receding_with_abort}, using soft constraints for both $x_r\in\hat{\Vs}$ and $x_N\in\hat{\Vs}$. The penalty weight on the receding constraint ($10^4$) is higher than the terminal one ($w_s = 10^2$) to mimic a hard constraint.
\end{itemize} 
For the simulations, we have considered a 3-joint manipulator, thus $n_x = 6,\, n_u = 3$. We have used \textsc{CasADi} \cite{Andersson2019} for the symbolic computation of the dynamics, costs and constraints, and \textsc{Acados}~\cite{Verschueren2019} to solve the OCPs and integrate the dynamics. The task is a setpoint regulation problem with respect to a state purposely chosen near the position limit of the first joint, to test the safety of the controllers:
\begin{equation}
    x^{\text{ref}} = (q^{\text{max}} - 0.05, \bar q, \bar q, 0, 0, 0),
    \label{eq:target_value}
\end{equation}
with $\bar q = (q^{\text{max}}+q^{\text{min}})/2$.
We have used as running cost a least-squares function, penalizing deviations from $x^{\text{ref}}$ and  control efforts:
\begin{equation}
    \begin{aligned}
        l(x, u) &= ||x-x^\text{ref}||_Q^2 + ||u||_R^2 \\
        Q &= \text{diag}([500, 10^{-4} I_5]), \quad R = 10^{-4} I_3,
    \end{aligned}
    \label{eq:weight_values}
\end{equation}
where $I_h$ is the identity matrix with size $h$. 
Set membership to $\hat{\Vs}$ is verified with the constraint:
\begin{equation} \label{eq:terminal_constraint}
(1-\alpha) \phi(x) - ||\dot{q}|| \ge 0,     
\end{equation}
where $\phi(\cdot)$ is a Neural Network (NN) computing an upper bound on the joint velocity norm~\cite{LaRocca2023}, and $\alpha \in [0, 1]$ is a safety margin that we introduced to ensure that $\hat \Vs \subseteq \Vs$. We have used \textsc{L4CasADi}~\cite{salzmann2023learning} for integrating the PyTorch~\cite{pytorch} neural model inside \textsc{Acados}. 

We have run 100 simulations for each MPC formulation, starting from the same 100 random joint positions $q_0$ with $\dot{q}_0 = 0$. The time step of the MPCs was $dt = \SI{5}{\milli\second}$. 
The horizon of the MPC has been fixed to $N=35$, 
so that each iteration takes less than $\SI{4}{\milli\second}$ (leaving $\SI{1}{\milli\second}$ for further operations, to mimic the timing limitations of a real-time application). 

\begin{table}[tb]
    \begin{center}
    \caption{Number of times each controller completed the task, safely aborted it, or violated a constraint (with $\alpha=2\%$).}
    \label{table:metrics_safety_2}
    {\renewcommand\arraystretch{1.3}
    \begin{tabular}{ cccc } 
    \toprule
    \sc{MPC} & \sc{Completed} & \sc{Aborted}  & \sc{Failed} \\ 
    \midrule
    \sc{Naive}  & 69 & - & 31 \\
    \sc{ST} & 69 & - & 31 \\
    \sc{STWA} & 70 & 11 & 19 \\
    \sc{HTWA} & 70 & 8 & 22 \\
    \sc{Receding} & 77 & 18 & 5  \\
    \bottomrule
    \end{tabular}
    }
    \end{center}
\end{table}


\begin{table}[tb]
    \begin{center}
    \caption{Number of times each controller completed the task, safely aborted it, or violated a constraint (with $\alpha=10\%$).}
    \label{table:metrics_safety_10}
    {\renewcommand\arraystretch{1.3}
    \begin{tabular}{ cccc } 
    \toprule
    \sc{MPC} & \sc{Completed} & \sc{Aborted}  & \sc{Failed} \\ 
    \midrule
    \sc{Naive}  & 69 & - & 31 \\
    \sc{ST} & 69 & - & 31 \\
    \sc{STWA} & 70 & 22 & 8 \\
    \sc{HTWA} & 70 & 21 & 9 \\
    \sc{Receding} & 77 & 20 & 3  \\
    \bottomrule
    \end{tabular}
    }
    \end{center}
\end{table}
\subsection{Small Safety Margin Tests}
Table~\ref{table:metrics_safety_2} reports the number of tasks completed, safely aborted, or failed by each controller, using a safety margin $\alpha=2\%$.
In terms of safety, \emph{Naive} and \emph{ST} violated the constraints the most. 
\emph{STWA} and \emph{HTWA} completed practically the same number of tasks as the previous methods, but in some cases were able to successfully abort the task, reducing the number of failed tasks. In general, \emph{Receding} performed the best, since it completed more tasks than the other formulations and safely aborted most of the remaining tests. Overall, \emph{Receding} failed only 5 tasks, which is $\approx$4 times less than the best competitor.

\subsection{Large Safety Margin Tests}
To try to reduce even more the number of failed tasks, we have carried out a second comparison with a higher safety margin $\alpha=10\%$ (see Table~\ref{table:metrics_safety_10}): 
the number of completed tasks is left unvaried, while the failures decreased (only 3 for \emph{Receding}). We can notice how a larger safety margin $\alpha$ increased the probability of being inside the real set $\mathcal{V}$, resulting in a higher success rate of the safe-abort OCP~\eqref{eq:safe_abort_ocp}. 

\begin{figure}[tbp]
    \centering
    \includegraphics[width=\columnwidth]{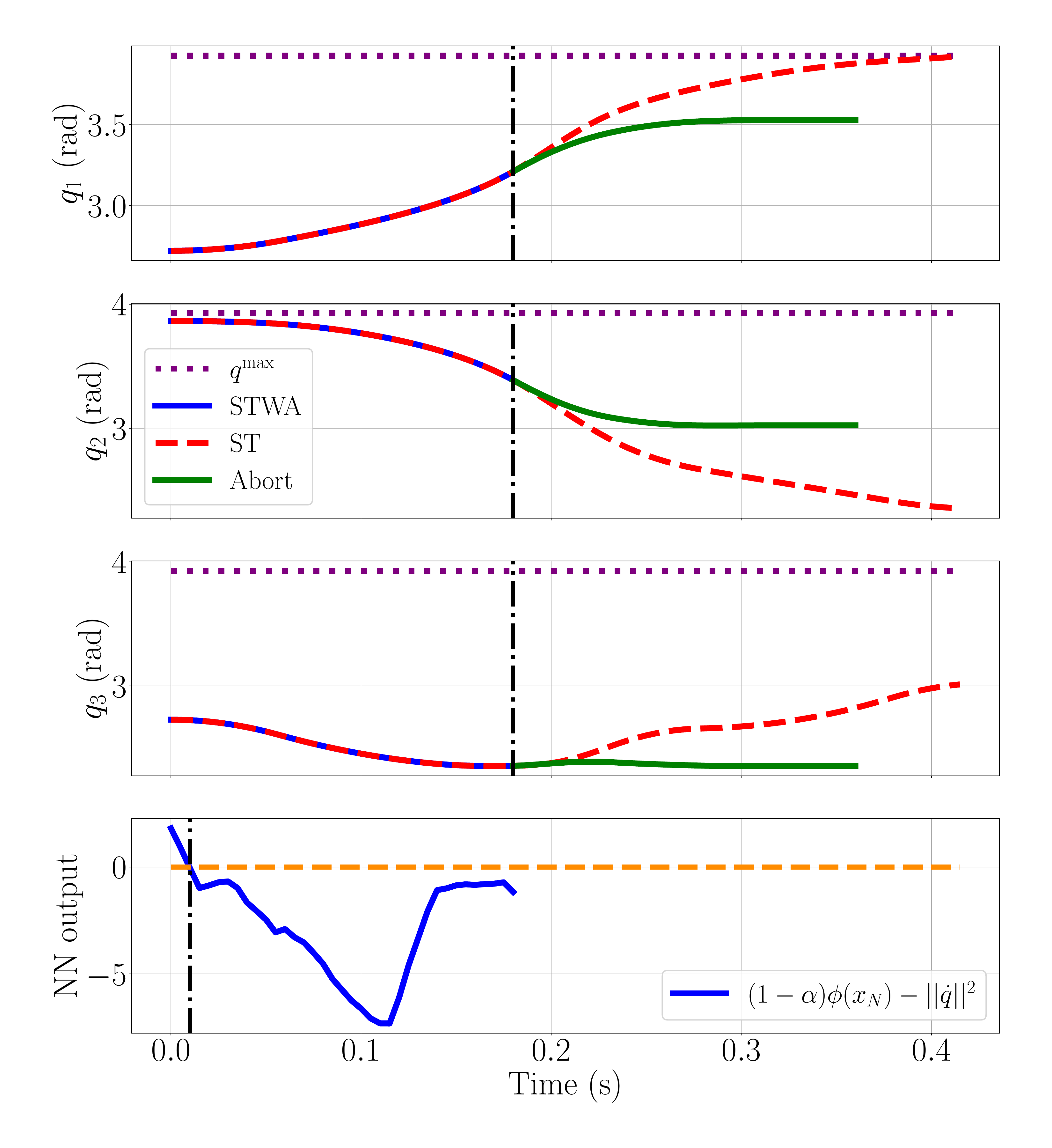}
    \caption{Comparison between \emph{ST} (task failed) and \emph{STWA} (task aborted). The bottom plot shows the value of the terminal constraint~\eqref{eq:terminal_constraint}. The vertical line highlights the start of the safe-abort trajectory.}
    \label{fig:abort_case}
\end{figure}
%
%
Fig.~\ref{fig:abort_case} highlights the different risk-aversion levels of \emph{ST} and \emph{STWA} by means of a specific simulation with $\alpha=10\%$. After a few time steps, the soft terminal constraint gets violated (NN output less than zero in the bottom plot). At this point, \emph{ST} is willing to take the risk, and keeps optimizing, which leads to $q_1$ hitting the position upper-bound at $t \simeq \SI{0.4}{\second}$. On the contrary, \emph{STWA}, triggers the task abortion as soon as the left-hand-side of~\eqref{eq:terminal_constraint} becomes negative; then, the system starts to follow a safe trajectory (green line) after $N - 1$ MPC steps. This shows how the safe abort can prevent a drastic failure. 


\subsection{Cost and Computation Time}
\begin{table}[tb]
    \begin{center}
    \caption{Mean tracking cost increment (with respect to Naive controller) and computation times for the MPC real-time iteration (RTI) and safe abort OCP.}
    \label{table:cost_times}
    {\renewcommand\arraystretch{1.3}
    \begin{tabular}{ ccccc } 
    \toprule
    \sc{MPC} & \sc{Cost $\uparrow$} ($\%$) & \sc{RTI} ($\SI{}{\milli\second}$) & \sc{Safe Abort} ($\SI{}{\second}$) \\
    \midrule
    \sc{Naive}    &    -    &  $3.75$ & - \\
    \sc{ST}       & $0.005$ &  $5.50$ & - \\
    \sc{STWA}     & $0.042$ & $3.73$ & $0.13$ \\
    \sc{HTWA}     & $0.042$ & $3.88$ & $0.10$ \\
    \sc{Receding} & $0.023$ & $3.95$ & $0.08$ \\
    \bottomrule
    \end{tabular}
    }
    \end{center}
\end{table}
In terms of cost, Table~\ref{table:cost_times} shows the average cost increment (in percentage) for the completed tasks (with $\alpha=10\%$). The increase is computed with respect to the \emph{Naive} mean cost through \eqref{eq:weight_values}. The percentages are very low for all the proposed formulations, thus the tracking performance is not degraded by the extra constraints using $\hat \Vs$. 

The same table also reports the computation times.
The 99-percentile for the Real-Time Iteration (RTI) scheme~\cite{diehl2005nominal} is always below the ideal computation time of $\SI{4}{\milli\second}$, apart from the \emph{ST} formulation, showing the benefits of our formulations also from this perspective. 
We should disclose that the RTI computation time of \emph{Receding} is not measured, but estimated: indeed, the Python interface of \textsc{Acados} does not support time-varying constraints. Therefore our current implementation of \emph{Receding} actually soft constrains the whole state trajectory in $\hat{\Vs}$, but then sets to zero the penalty weights for all time steps except for $r$ and $N$, resulting in a much higher computation time than needed. Thus we have estimated the solver computation time as the sum of the time to solve the Quadratic Program of \emph{Receding} ($\SI{3.47}{\milli\second}$) and the time to compute a single constraint linearization ($\SI{0.48}{\milli\second}$), taken from \emph{STWA}. We believe this estimate to be reasonably accurate, and since it is lower than $dt$, we think that a proper implementation of \emph{Receding} would work in real-time scenarios. 

The last column reports the maximum computation times for the Task Abortion. All the values satisfy Assumption~\ref{ass:ocp_time}. The backup OCP for \emph{Receding} should be solved in one time-step, which is not possible with these computation times. Anyway, possible strategies can be applied to deal with the computation requirements, as discussed in Section~\ref{ssec:abort_receding}.

\section{Conclusions}
We have presented a novel Receding-Constraint MPC formulation, which provides recursive feasibility guarantees under a weaker assumption on the used safe set with respect to classic approaches. 
Moreover, we have presented a task-abortion strategy that allows to reach an equilibrium state whenever a risk of constraint violation is detected. 
Our results on a 3-joint manipulator show the improved safety of the presented Receding-Constraint MPC with respect to other state-of-the-art methods.

While our method relies on weaker assumption than standard approaches, these assumptions are still hard to verify in practice, which can lead to constraint violation, as shown in our tests. Inspired by recent work~\cite{Robey2020}, we want to investigate the use of robust optimization techniques to certify N-Step Control Invariance.
Future research will focus also on reducing the computation times of the safe-abort OCP~\eqref{eq:safe_abort_ocp} to make it usable with the Receding-Constraint MPC. For this, we plan to extend the method in~\cite{LaRocca2023} to learn both the set $\hat{\Vs}$ and a policy that drives the state to an equilibrium. Then the policy can be used to warm-start the solver. 
We also plan to account for uncertainties in the dynamics using robust optimization, and to include Cartesian constraints to avoid collision with static obstacles.  
While this work focused on model-based control methods, our approach could be applied in the future to \emph{safety filters} for making black-box RL policies safe.

\addtolength{\textheight}{-9cm}   


\bibliographystyle{IEEEtran}
\bibliography{IEEEabrv,main.bib}

\end{document}

%% file: math_commands.tex
\DeclareMathOperator*{\minimize}{minimize}				
\DeclareMathOperator*{\st}{subject\,to}					

\newcommand{\Xs}[0]{\ensuremath{\mathcal{X}}}						
\newcommand{\Us}[0]{\ensuremath{\mathcal{U}}}						
\newcommand{\Vs}[0]{\ensuremath{\mathcal{V}}}						

\newcommand{\Expect}{{\rm I\kern-.3em E}}				

\newenvironment{definition}[1][Definition]{\begin{trivlist}
\item[\hskip \labelsep {\bfseries #1}]}{\end{trivlist}}

\newtheorem{theorem}{Theorem}
\newtheorem{lemma}{Lemma}

\newtheorem{assumption}{Assumption}

\makeatletter
\renewenvironment{proof}[1][\proofname]{\par
\pushQED{\qed}%
\normalfont \topsep6\p@\@plus6\p@\relax
\trivlist
\item\relax
{\itshape
#1\@addpunct{.}}\hspace\labelsep\ignorespaces
}{%
\popQED\endtrivlist\@endpefalse
}
\makeatother